\documentclass[11pt]{article}

\usepackage{float}
\usepackage{array, makecell} %
\usepackage{xcolor,colortbl}
\usepackage{framed,color}
\usepackage{xcolor}
\usepackage{enumitem}
\usepackage{thm-restate}
\usepackage[numbers,sort&compress]{natbib}
\usepackage{booktabs}

\usepackage{amsthm,amsmath,amssymb}
\usepackage[bookmarks=true,hypertexnames=false,pagebackref]{hyperref}
\hypersetup{colorlinks=true, citecolor=blue, linkcolor=red,
  urlcolor=blue}
\usepackage{newpxtext}
\usepackage[libertine,vvarbb]{newtxmath}

\usepackage[T1]{fontenc} % best for Western European language
\usepackage{textcomp} % required to get special symbols

\usepackage[scr=rsfso]{mathalfa}
\usepackage{nicefrac}
\usepackage{cleveref}
\usepackage{graphicx}
\usepackage{aliascnt}
\usepackage[section]{placeins}
\usepackage{tcolorbox}
\usepackage{epigraph}
\usepackage{algorithm}
\usepackage{algpseudocode}
\usepackage{xifthen}
\usepackage{latexsym}
\usepackage{framed}
\usepackage{fullpage}
\usepackage{bm}
\usepackage{braket}
\usepackage{tabularx}

\usepackage{hyperref}
\usepackage{url}

\newtheorem{theorem}{Theorem}

\newtheorem{remark}{Remark}
\newtheorem{proposition}{Proposition}

\newtheorem{lemma}{Lemma}

\newcommand{\eps}{\eta}
\newcommand{\error}{S}
\renewcommand{\vec}[1]{\boldsymbol{#1}}

\newcommand{\KL}{D_{\mathrm{KL}}}
\newcommand{\E}{\mathbb{E}}

\newcommand{\M}{\mathcal{M}}
\newcommand{\X}{\mathcal{X}}
\newcommand{\Hellinger}{\mathrm{Hel}}

\title{Approximate Solomonoff Induction: On the Robustness of Learned Universal Predictors}

\author{Qian Li \\
Shenzhen Research Institute of Big Data \\
\texttt{liqian.ict@gmail.com} \\
\and
Shang-Hua Teng \\
University of Southern California\\
\texttt{shanghua@usc.edu}
}

\date{}

\begin{document}

\maketitle

\begin{abstract}
Why large language models trained via log-loss minimization over massive and diverse datasets exhibit broad capabilities and generalizable skills is not yet fully understood. Grau-Moya et al. (ICML 2024) offer a conceptual explanation through the lens of meta-learning: if sequence-generating environments $\nu$ are drawn from a prior $w_{\nu}$, minimizing population log-loss recovers the corresponding Bayesian mixture predictor $\xi=\sum_\nu w_\nu\nu$. In the idealized limit in which $w$ is a universal prior, the optimal predictor corresponds to Solomonoff universal induction. Practical machine learning algorithms, however, generally produce only \emph{approximately} optimal predictors. This raises a natural question: \emph{Does the prediction performance remain robust under approximation}? In this paper, we first provide a tight bound on the worst-case effect of approximation on prediction error. We show any predictor $\rho$ satisfying $D_{KL}(\xi \|\rho)\leq \eta$ retains bounded cumulative-expected-squared prediction error. However, while the Bayesian mixture predictor $\xi$ has error $O(\log(1/w_\mu))$ in an environment $\mu$, we establish a tight $\Theta(\eta/w_\mu)$ bound on the excess error of $\rho$ relative to $\xi$. Thus, approximation can substantially degrade performance in environments assigned small prior probability. In sharp contrast, when environments are averaged according to prior $w$, we prove that the excess prediction error is only $O(\eta)$. Furthermore, we establish a distribution-shift theorem for dominated shifts from a training prior $w$ to a test prior $w'$, interpolating our average-case and worst-case guarantees. Finally, we extend our robustness analysis to a finite-sampling setting relevant to learning universal predictors with neural networks and Transformers, where $D_{KL}(\xi\|\rho)$ could be unbounded. Our theory extends the classical theory of Bayesian and Solomonoff Induction to approximate predictors and characterizes a broad spectrum of prediction behavior. It helps explain the sharp gap between strong empirical benchmark performance and  poor worst-case  performance in rare environments, while providing quantitative justification for %targeted sampling, 
fine tuning and synthetic data generation 
designed to protect important but underrepresented task families.
\end{abstract}

\section{Introduction}

Large language models (LLMs) trained for \emph{sequential prediction} have acquired a remarkable range of capabilities from vast and diverse data. In recent years, scaling model capacity, expanding the volume and breadth of training data, increasing computational resources \citep{DBLP:journals/corr/abs-2001-08361}, and incorporating real-world human interaction \citep{ouyang2022training} have driven substantial progress, transforming LLMs into increasingly general-purpose systems for %prediction, 
problem solving and induction. Given the substantial and rapidly growing resources required to train frontier models, understanding the benefits and limitations of further scaling has become more important than ever, raising a range of fundamental research questions.

Classical learning frameworks, such as  semi-supervised learning, PAC learning, and reinforcement learning, typically analyze learning relative to a specified task, data distribution, and hypothesis class.
In contrast, a single LLM trained on heterogeneous data can perform a wide range of tasks—including translation, writing, coding, and mathematical reasoning—without %or with only modest 
task-specific retraining, guided only by instructions or a few demonstrations provided in context.
Although pretraining is largely driven by minimizing next-token log loss, the resulting models exhibit capabilities involving generation, abstraction, reasoning, and planning, often without explicit supervision for each capability. Moreover, once trained, such a model can be deployed at a comparatively low marginal cost per inference \citep{deepseekai2025}. LLM training thus distills enormous amounts of data and computation into a reusable neural network whose fixed parameters support adaptive behavior across contexts and tasks. This phenomenon suggests that learning to predict sufficiently rich and heterogeneous sequences can give rise to broad knowledge and generalizable computational abilities. Precisely how such capabilities emerge from a seemingly simple prediction objective, however, remains incompletely understood,
and is central to characterizing the fundamental power and limitations of the current LLM paradigm.

\paragraph{Motivating Prior Work.} \emph{Meta-learning}, or learning to learn, provides a conceptual framework for understanding this phenomenon \citep{entropy2026}. Meta-learning treats the learning procedure itself as an object of optimization, using experience accumulated across tasks to improve the learning of new tasks \citep{meta-learning-survey}. It involves two levels of learning, either explicitly or implicitly: an inner process learns a particular task from a specified dataset and objective, while an outer process optimizes aspects of the inner process--such as the model class, initialization, or optimization rule--with the aim of reducing the data or computation required to solve new tasks.
 
 %In contrast with traditional machine learning where tasks are solved from scratch using a fixed learning algorithm, meta-learning aims at improve the learning algorithm itself, by leveraging the experience across multiple learning episodes  

%One main type of meta-learning is termed \emph{memory-based meta-learning} (MBML): memory-augmented models are trained using a single optimization loop and a diverse dataset containing data from different tasks. The memory is necessary to enable task adaptation at test time. while adaptation to an individual task occurs through its context or internal state, with those parameters held fixed. The two loops are thus implicit and implemented by sampling tasks (outer loop) and data from tasks (inner loop). At test time, there is no inner gradient optimization but automatic in-context learning driven by the memory dynamics. Note that Transformers directly instantiate the memory in the context. 

One approach is \emph{memory-based meta-learning} (MBML) \citep{MetaLearningWithMemoryAugmentedNN}, in which a memory-augmented model--an LLM whose context serves as memory, for example--is trained on data drawn from a distribution over tasks using a single optimization loop. The two levels of learning are realized through different mechanisms: the outer process updates the model’s parameters across training tasks, whereas the inner process adapts to an individual task through changes in the model’s memory state. At test time, the parameters remain fixed, and adaptation proceeds through memory dynamics without task-specific parameter updates. In-context learning thus emerges as a learned adaptation procedure implemented through the model’s forward computation.

%\emph{Meta-learning}, or learning to learn, provides a conceptual framework for understanding this phenomenon: it uses experience across tasks to acquire a learning strategy that enables adaptation to new tasks with less data or computation \citep[see][for a survey]{meta-learning-survey}. In memory-based meta-learning, the model’s parameters are learned across tasks, while adaptation to an individual task occurs through its context or internal state, with those parameters held fixed. 

This MBML framework admits a precise Bayesian characterization \citep{ortega2019meta,DBLP:conf/iclr/XieRL022}. For a fixed horizon $n$, suppose that each training sequence is generated by first sampling a task $\nu$ from a prior distribution $w$, then sampling data $x_{1:n}$ from $\nu$. Suppose a sufficiently expressive model $\rho_{\theta^*}$ is trained to minimize the population log-loss objective
\[
\mathcal L_n(\rho_{\theta})=\mathbb E_{\nu\sim w}\mathbb E_{x\sim\nu} \left[-\sum_{t=1}^n\log_2 \rho_{\theta}(x_t\mid x_{<t})\right]
\]
Then $\rho_{\theta^*}$ behaves as the Bayesian mixture predictor $\xi(x)=\sum_{\nu}w_{\nu}\cdot \nu(x)$. 
In particular,
\[
\rho_{\theta^*}(x_{t}\mid x_{<t})=\xi(x_t\mid x_{<t})=\sum_{\nu}w_\nu(x_{<t})\nu(x_t\mid x_{<t}),
\]
where $w_{\nu}(x_{<t})\propto w_{\nu}\cdot \nu(x_{<t})$ is the posterior probability of $\nu$. 
Thus, conditioned on an observed prefix $x_{<t}$, the Bayesian predictor combines the predictions of the candidate tasks according to their posterior probabilities, reflecting both their prior weights and their consistency with the observations. 
The appeal of the Bayesian perspective lies in its strong prediction guarantee \citep{uai}: at test time, $\rho_{\theta^*}$ can achieve efficient in-context adaptation to any task $\mu$ with positive prior weight $w_{\mu}>0$:
\begin{align*}
    \error_n(\mu,\rho_{\theta^*}):=\sum_{t=1}^n \mathbb{E}_{x_{<t}\sim \mu}\sum_{x_{t}}\left(\mu(x_{t} \mid x_{<t})-\rho_{\theta^*}(x_t\mid x_{<t})\right)^2=O\left(\log \frac{1}{w_\mu}\right).
\end{align*}
%This provides a principled explanation of in-context learning as \emph{amortized Bayesian prediction}: training encodes an adaptive inference procedure into the forward computation of a neural network.
This provides a principled explanation of in-context learning as \emph{amortized Bayesian prediction}: training amortizes learning across tasks by encoding an adaptive inference procedure into the forward computation of a neural network.
In particular, consider the class of all computable environments\footnote{We use the terms ``task'' and ``environment'' interchangeably.} endowed with the universal prior  $w_\nu=\Theta(2^{-K(\nu)})$, where $K(\nu)$ denotes the Kolmogorov complexity of $\nu$. 
The resulting \emph{single} Bayesian predictor becomes Solomonoff universal induction \citep{solomonoff1978complexity,solomonoff1986application}, which achieves cumulative expected squared prediction error $O(K(\mu))$ simultaneously in \emph{every} computable environment $\mu$, independently of the prediction horizon. Building on this connection, \citep{MGH+24} proposed learning increasingly general predictors by meta-training neural networks on synthetic data generated by computationally bounded environments, represented by universal Turing machines with bounded program lengths and runtimes. Progressively relaxing these bounds while increasing model capacity and training data offers an principled approach toward universal prediction, provided that the corresponding prediction objectives can be optimized sufficiently well. The paper \citep{entropy2026} placed this framework within the broader context of \emph{algorithmic compression}: minimizing sequential log loss can be understood as learning to compress observations by approximating and amortizing Bayesian prediction over their underlying generative mechanisms. Under a universal prior, the resulting predictor favors simpler computational explanations while averaging over all computable hypotheses according to both their simplicity and their consistency with the observations.
%In particular, when we consider all computable environments\footnote{We use the terms "task" and "environment'' interchangeably.} with universal prior $w_\nu=2^{-\Theta(K(\nu))}$, where $K(\mu)$ is the Kolmogorov complexity of the environment, the corresponding \emph{single} Bayesian predictor implements the Solomonoff universal induction, and thus enjoys a finite cumulative expected prediction error bound $O(K(\nu))$ simultaneously in \emph{every} computable environment. Based on this connection, \citet{MGH+24} claimed that: though Solomonoff induction is uncomputable, we in principle can approach universal prediction by meta-training larger and larger neural networks solely on more and more synthetic data generated by universal Turing machines. \citet{entropy2026} further placed this approach within a broader account of algorithmic compression, based on the property that Bayesian favors simple computational explanations while retaining every computable explanation consistent with the observations.
%Building on this connection, \citet{MGH+24} proposed approximating this uncomputable ideal by meta-training neural networks on synthetic data generated by universal Turing machines. Their framework offers an in-principle approach toward universal prediction as model capacity and data increase, provided that the corresponding prediction objectives can be optimized sufficiently well. 
\paragraph{Our Main Contributions.} The classical theory \citep{solomonoff1978complexity,solomonoff1986application,uai} focuses on the exact predictor that attains the minimum population log loss. In practice, however, limited model capacity and imperfect optimization--even when training data are abundant--mean that learned models generally only approximate the ideal Bayesian predictor. As emphasized by \citep{entropy2026}, extending the Bayesian view of meta-learning beyond exact realizability and perfect optimization remains an important open direction.  Motivated by this line of work, we take a first step toward understanding the fundamental gap between ideal Bayesian prediction and practical learned predictors by addressing the following central question:
\begin{center}
\textit{{To what extent do the guarantees of the ideal Bayesian predictor continue to hold when practical learning architectures can only approximate it?}}
\end{center}
%\textbf{Shanghua: I am here}

In this work, we develop a theory of Bayesian and universal sequential prediction under approximation, characterizing performance and robustness across a broad range of regimes:

{\textit{Environment-Wise Worst-Case Robustness.}}
We show in Theorem~\ref{thm:upper} that every approximate predictor $\rho$ with excess population log loss at most $\eta$, equivalently $D_{\mathrm{KL}}(\xi\Vert\rho)\leq\eta$, retains a finite bound on its cumulative expected squared prediction error. Specifically, for any environment $\mu$, the error of the ideal Bayesian predictor $\xi$ is $O(\log(1/w_\mu))$, whereas that of $\rho$ is bounded by
$O\left(\log{1}/{w_\mu}+{\eta}/{w_\mu}\right)$.
We further show in Theorem~\ref{lem:lower_universal} that the approximation-dependent term is tight. In particular, for a universal mixture, an error of order $\eta 2^{K(\mu)}$ is attained for infinitely many computable environments $\mu$. 
We also extend the analysis to two other fundamental measures of prediction error: cumulative KL divergence and cumulative squared Hellinger distance.

{\textit{Typical Average-Case Robustness:}}
In sharp contrast, we prove in Theorem~\ref{thm:upper2} that, when performance is averaged over environments according to the prior $w$--the canonical average-case measure--the excess prediction error due to approximation is only $O(\eta)$. We show that an analogous average-case bound holds for cumulative KL divergence but not for cumulative squared Hellinger distance.

{\textit{Distribution-Shift Robustness:}} Furthermore, we establish a distribution-shift theorem (Theorem~\ref{thm:upper3}) for dominated shifts from a training prior $w$ to a test prior $w'$, where $w'_\nu\leq Cw_\nu$ for every environment $\nu$ and some $C\geq 1$. The resulting error bound is $O(\log C+C\eta)$, thereby interpolating between our average-case and worst-case guarantees.

%{\textit{Finite-Data Robustness:}} \warn{Finally, in Section \ref{sec:intro_data}, we study the finite-data setting relevant to learning universal predictors with neural networks and Transformers, and characterize the fundamental impossibility for a trained model $\rho_m$ learned from  $m$ training sequences to achieve a bounded $D_{KL}(\xi\|\rho_m)$.  We identify a weaker alternative measure of approximation, extend our robustness analysis under this measure to the finite-data setting, and show that it may be better suited to addressing the finite-sample nature of practical statistical learning.}

{\textit{Finite-Data Robustness:}}
Finally, in Section~\ref{sec:intro_data}, we study predictors learned from finitely many training sequences. We show exponential upper-tail bounds for population KL error, yet its expectation can remain infinite when unseen events receive zero probability. Nevertheless, by using a weaker alternative measure of approximation, we extend our robustness analysis to this setting. In particular, we bound the expected average excess prediction error by the representation and optimization errors plus a finite-sample term that decays exponentially with the number of training sequences.
%Nevertheless, by using a weaker alternative measure of approximation, we extend the above robustness analysis. In particular, it turns out the finite-data contributes an exponentially decaying finite-sample term to the average excess prediction error, additional to representation and optimization errors.

%\begin{itemize}
%\item  
%\item 
%\item To mitigate worst-case degradation, we propose using task tags during training to enable separate control of task-specific approximation errors, without requiring tags at test time (Section \ref{sec:intro_mitigation}).
%\item 
%\end{itemize} 

Taken together, these results provide a broad characterization of robustness to approximation in sequential prediction. They explain the sharp contrast between strong aggregate benchmark performance and poor worst-case performance on rare environments. This sharp distinction between worst-case and average-case behavior helps explain why approximation may have a much milder effect in practice. It may also shed light on the strong empirical performance of learned universal predictors \citep{MGH+24,entropy2026} and, more broadly, of LLMs across a wide range of tasks and benchmarks.
Our distribution-shift theorem has several implications beyond basic transfer principle for sequential prediction. It provides a quantitative rationale for targeted sampling and synthetic data generation designed to protect important but underrepresented families of tasks. In particular, it provides a theoretical rationale for expecting that augmenting real-world datasets with Universal-Turing-Machine-generated (UTM-generated) data may improve prediction performance on practically important task families, as suggested by~\cite{MGH+24}.

\section{Notations and Preliminaries}\label{sec:prelim}
We assume familiarity with the basic notions in probability theory and algorithmic information theory, including conditional probability and Turing machines. See, for example, \citep{li2008introduction} for background. Throughout this paper, $\log$ denotes the base-$2$ logarithm, while $\ln$ denotes the natural logarithm. Entropy and KL divergence are measured in bits.

\paragraph{Sequence measures and mixtures.}
Let $\X$ be a finite alphabet with $|\X|\geq 2$.
Suppose $\mu$ is a  probability measures on $\X^\infty$ (representing a sequence generative model). 
We use $\mu(x_{1:n})$ to denote the probability that a sequence sampled from $\mu$ begins with $x_{1:n}$, and $\mu_{1:n}$ to denote the induced distribution on $\X^n$. We denote $x_{<t}:=x_{1:t-1}$. Let $\M=\{\nu_1,\nu_2,\ldots\}$ be a countable class of measures. 
Let 
$\{w_i>0\}$ be weights satisfying $\sum_iw_i=1$. 
Then, define the mixture
$\xi:=\sum_{i=1} w_i\nu_i.$
When $\mu=\nu_i$ is a selected measure from $\M$, we write $w_\mu:=w_i$. Note that for
every finite string $x$,
$\xi(x)\geq w_\mu\mu(x)$.

\paragraph{The universal mixture.} 
Solomonoff \cite{solomonoff1978complexity,solomonoff1986application} introduced a universal prior for \textit{sequence prediction}. Unless stated otherwise, we fix a universal monotone Turing machine $U$ throughout the paper. We say a measure $\nu$ is \textit{estimable} if there exists a monotone Turing machine that
%, takes a infinite sequence of i.i.d. fair coin tosses, 
takes an infinite sequence of independent, uniformly random bits and generates a sequence according to $\nu$. Let $K(\nu)$ denote the length of the shortest program for $\nu$ on this machine.

Let $\M_U$ denote the class consisting of all estimable measures. Set
\[
Z:=\sum_{\nu\in\M_U}2^{-K(\nu)},
\qquad
w_\nu:=\frac{2^{-K(\nu)}}{Z},
\qquad
\xi_U:=\sum_{\nu\in\M_U}w_\nu\nu.
\]
The programs are described in a prefix-free fashion, so Kraft's
inequality gives $0<Z\leq1$, and $w_\nu=\Theta(2^{-K(\nu)})$. 
The fair-coin measure belongs to $\M_U$, so $\xi_U(x)>0$ for every finite
string $x$. We also use $K(x)$ for the usual prefix-free Kolmogorov  complexity of a finite
string or integer $x$.

\paragraph{KL divergence.}
For two measures $\mu$ and $\rho$ and $n\geq1$, define
\[
D_{\mathrm{KL}}^n(\mu\Vert\rho)
:=\sum_{x_{1:n}\in\mathcal X^n}
\mu(x_{1:n})\log\frac{\mu(x_{1:n})}{\rho(x_{1:n})},
\qquad
D_{\mathrm{KL}}(\mu\Vert\rho)
:=\lim_{n\rightarrow\infty}D_{\mathrm{KL}}^n(\mu\Vert\rho).
\]
We use the convention that $0\log(0/q)=0$ and $p\log(p/0)=+\infty$ for $p>0$. The finite-horizon
divergences are nondecreasing in $n$, so the supremum is also their limit.  The chain rule gives that:
\begin{equation}\label{eq:chainrule}
D_{\mathrm{KL}}^n(\mu\Vert\rho)
=\sum_{t=1}^n\E_{x_{<t}\sim\mu}
\left[
D_{\mathrm{KL}}\!\left(
\mu(\,\cdot\mid x_{<t})
\Vert
\rho(\,\cdot\mid x_{<t})
\right)
\right].
\end{equation}

\paragraph{Cumulative prediction error.}
As in \citep{solomonoff1978complexity}, we use \emph{cumulative squared prediction error} as our primary criterion for evaluating the prediction performance of $\rho$:
%We use \emph{cumulative squared prediction error} as our primary criterion for evaluating the prediction performance of $\rho$:
\[
\error_n(\mu,\rho)
:=
\sum_{t=1}^n
\mathbb{E}_{x_{<t}\sim\mu}
\sum_{a\in\X}
\left(
\mu(a\mid x_{<t})-\rho(a\mid x_{<t})
\right)^2,
\qquad
\error(\mu,\rho):=\lim_{n\rightarrow\infty}\error_n(\mu,\rho).
\]
To examine the scope of our guarantees across different error
criteria, we also consider \emph{cumulative squared Hellinger error}:
\[
\Hellinger_n(\mu,\rho):=\sum_{t=1}^n\mathbb{E}_{x_{<t}\sim\mu}
\sum_{a\in\X}\left(\sqrt{\mu(a\mid x_{<t})}-\sqrt{\rho(a\mid x_{<t})}\right)^2, \qquad
\Hellinger(\mu,\rho):=\lim_{n\rightarrow\infty}\Hellinger_n(\mu,\rho).
\]
and \emph{cumulative KL error}:
\[
D_n(\mu,\rho)
:=
\sum_{t=1}^n
\mathbb{E}_{x_{<t}\sim\mu}
\left[
D_{\mathrm{KL}}\!\left(
\mu(\cdot\mid x_{<t})
\Vert
\rho(\cdot\mid x_{<t})
\right)
\right]
=
D_{\mathrm{KL}}^n(\mu\|\rho),\qquad
D(\mu,\rho):=\lim_{n\rightarrow\infty} D_n(\mu,\rho).
\]
The following inequalities relate these error criteria.

\begin{lemma}[Theorem 3.2.3 in \citep{uai2}]
\label{lem:pinsker}
Let $\mu$ and $\rho$ be probability measures on $\X^\infty$.
For every $n\geq1$,
\begin{align*}
\error_n(\mu,\rho)\leq (\ln 2)\,D_n(\mu\Vert\rho),\qquad \Hellinger_n(\mu,\rho)\leq (\ln 2)\,D_n(\mu\Vert\rho),\qquad
\error_n(\mu,\rho)\leq 2\,\Hellinger_n(\mu,\rho).
\end{align*}
The same inequalities hold for the corresponding infinite-horizon
errors $\error$, $\Hellinger$, and $D$.
\end{lemma}

\section{Robustness under Approximation: Task-wise Guarantee}\label{sec:results}
%We study predictor $\rho$ that are approximately optimal, as measured by the excess population log-loss. We study how approximation in the population log-loss affects prediction performance. 
In this section, we establish a tight task-wise prediction guarantee for a trained model $\rho$
that approximates the Bayesian predictor  $\xi$ with KL gap $\KL(\xi\|\rho)\leq\eta$. 
The KL divergence provides the most natural measure of approximation quality because the 
population log-loss of $\rho$ at horizon $n$ admits the following decomposition, where $H$ denotes the \textit{Shannon entropy}:
\[
\mathcal{L}_n(\rho)
:=
\mathbb{E}_{x_{1:n}\sim\xi}
\left[-\log\rho(x_{1:n})\right]
=
H(\xi_{1:n})+D_{\mathrm{KL}}^n(\xi\Vert\rho),
\]
Moreover, $\xi$ minimizes the population loss over all probability measures on $\X^{\infty}$, attaining the minimum value $H(\xi_{1:n})$. Thus,
$\KL(\xi\|\rho)\leq\eta \ \Longrightarrow \ D_{\mathrm{KL}}^n(\xi\Vert\rho)\leq\eta
\ \Longleftrightarrow\ 
\mathcal{L}_n(\rho) 
\leq
\min_{\pi}\mathcal{L}_n(\pi)+\eta.
$

%The proofs are provided in Appendices A–C. 

\subsection{Environment-Wise Guarantee: Upper Bound}\label{sec:good}

We first extend the classical analysis of \citep{uai} by showing that, under a finite KL-divergence, the cumulative squared prediction error remains finite for every environment in $\M$.
%We show that such an approximation still enjoys finite cumulative prediction error.
\begin{theorem}[Environment-wise upper bound]\label{thm:upper}
If $\KL(\xi\|\rho)\leq\eps$, then for every $\mu\in\mathcal M$,
\[
\error(\mu,\rho)\leq\frac{e\ln 2}{e-1}\left(\log\frac{1}{w_\mu}+\frac{\eps}{w_\mu}\right).
\]
\end{theorem}
The proof is given in Appendix \ref{Appendix:ProofTaskWiseUpperBound}, which 
bounds the cumulative KL error
for every $n$ by
\[
D_n(\mu\Vert\rho)
\leq
\frac{e}{e-1}
\left(
\log\frac{1}{w_\mu}+\frac{\eps}{w_\mu}
\right).
\]
Theorem \ref{thm:upper} then follows from Lemma~\ref{lem:pinsker} that 
$\error_n(\mu,\rho)\leq (\ln 2)\,D_n(\mu\Vert\rho)$ and $\Hellinger_n(\mu,\rho)\leq (\ln 2)\,D_n(\mu\Vert\rho)$.
We thus also prove that, up to constant factors, this upper bound extends the cumulative squared Hellinger distance $\Hellinger(\mu,\rho)$. 

Specializing to the universal mixture $\xi_U$ and using the relation $w_\mu=\Theta(2^{-K(\mu)})$, we obtain the following environment-wise prediction-error bound in terms of the Kolmogorov complexity $K(\mu)$.

\begin{theorem}[Upper bound for universal prediction]\label{cor:universal-upper}
 If $\KL(\xi_U\|\rho)\leq \eps$, then for every $\mu\in\M_U$,
\[
\error(\mu,\rho)=O\left(K(\mu)+\eps\cdot 2^{K(\mu)}\right),
\] 
where the implicit constant depends only on the fixed universal Turing machine.
\end{theorem}

%Consequently, Corollary~\ref{cor:universal-upper} also holds with $\error$ replaced by $D$ or $\Hellinger$.

\subsection{Environment-Wise Guarantee: Tight Worst-Case Lower Bound}\label{sec:bad}

Our next theorem shows that the approximation-dependent term $\eps/w_{\mu}$  in Theorem~\ref{thm:upper} is unavoidable and tight up to a constant factor. Remarkably, for the universal mixture, this tightness holds for infinitely many environments in $\M_{U}$. Throughout this subsection, we restrict attention to the binary alphabet $\X=\{0,1\}$.
 % the following result matches the upper bound in \warn{Theorem~\ref{cor:universal-upper}}, up to constant factors,
\begin{theorem}[Tight lower bound for universal prediction]
\label{lem:lower_universal}
For every $\eps>0$, there are infinitely many estimable measures
$\mu$ such that
\[
\sup_{\rho:\,D_{\mathrm{KL}}(\xi_U\Vert\rho)\leq\eps}\error(\mu,\rho)\geq\Omega\left(K(\mu)+\eps\,2^{K(\mu)}\right),
\]
where the implicit constant depends only on the fixed universal Turing machine.
\end{theorem}

Thus, even an arbitrarily small KL-approximation gap can coexist with a cumulative squared prediction error that grows exponentially with the Kolmogorov complexity of the environment.

We now describe the high-level idea behind the proof. The complete proof is given in Appendix~\ref{AppendixProofLowerBound} and formalizes the following rare-prefix construction.

\paragraph{Rare-prefix construction.}
The lower bound exploits prefixes that are unlikely under the Bayesian mixture $\xi$, but occur with probability one under a particular environment. For $z\in\{0,1\}^{\infty}$, let $\delta_z$ denote the deterministic environment that outputs $z$. Suppose $\mu=\delta_z$ and that a finite prefix $y$ of $z$ has mixture probability $q:=\xi(y)$. We construct $\rho$ by replacing the next $N$ conditional predictions along $z$ after the prefix $y$ with the uniform distribution, while leaving all other predictions unchanged.

Each modified history has probability at most $q$ under $\xi$, and replacing its next-symbol distribution with the uniform
distribution incurs at most one bit of conditional KL divergence. Hence, the total KL cost is at most $qN$.
Under $\mu$, however, every modified history is visited with probability one, and each modified prediction contributes exactly $1/2$ to the squared prediction error. Taking $N=\lfloor\eps/q\rfloor$ yields the following lemma.

\begin{lemma}[Lower bound from a rare prefix]\label{lem:lower}
Let $\mu=\delta_z\in\mathcal M$, and let $y\prec z$ be a finite prefix of $z$ and $q:=\xi(y)$. For every $\eps>0$, there exists a probability measure $\rho$ with $D_{\mathrm{KL}}(\xi\Vert\rho)\leq\eps$ such that
\[
\error(\mu,\rho)\geq\error_{|y|}(\mu,\xi) + \frac{1}{2} \left\lfloor\frac{\eps}{q}\right\rfloor.
\]
\end{lemma}
When $q=O(w_\mu)$, it yields an $\Omega(\eps/w_\mu)$ lower bound whenever $\eps/w_\mu$ is sufficiently large. 
The proof of Theorem~\ref{lem:lower_universal} identifies infinitely many deterministic environments having prefixes
that satisfy this condition. Since $w_\mu=\Theta(2^{-K(\mu)})$, the resulting approximation penalty is $\Omega(\eps\,2^{K(\mu)})$. 
%For each fixed $\eps>0$, this term also dominates $K(\mu)$ at sufficiently large complexities, giving the stated bound.
%\begin{remark}[Other error criteria]
%\label{rem:lower-other-errors}

Finally, we can show that Lemma \ref{lem:lower} also extends to cumulative KL and squared Hellinger errors. Then by Lemma~\ref{lem:pinsker}, we have, for every $n\geq1$,
$D_n(\mu\Vert\rho) \geq {\error_n(\mu,\rho)}/{\ln 2}$ and 
$\Hellinger_n(\mu,\rho)\geq \error_n(\mu,\rho)/2.$
Thus, our construction can establish similar lower bounds as in Theorem~\ref{lem:lower_universal}  for cumulative KL and squared Hellinger errors.
%Indeed, 
%Taking the supremum over $n$ shows that still holds with $\error$ replaced by $D$ or $\Hellinger$. 
%\end{remark}

\section{Robustness under Approximation: Average-Case Analysis}\label{sec:average}

In this section, we show that the environment-wise prediction guarantee
under approximation can be sharply improved when performance is averaged
across environments. We then establish a distribution-shift theorem for
moving from the training prior $w$ to a dominated test prior $w'$. The
result interpolates between our average-case and environment-wise guarantees.
%These constitute the main results of the paper.
Before stating the theorems and discussing their implications, we explain
the motivation for developing an average-case theory of robustness under
approximation. Section~\ref{sec:bad} shows that the approximation-dependent
term $\eps/w_\mu$ in Theorem~\ref{thm:upper} is unavoidable in the worst case.
For an environment with a small prior weight $w_\mu$, this term can overwhelm
the classical $O(\log(1/w_\mu))$ prediction guarantee for the ideal Bayesian
predictor.

Such pessimistic worst-case guarantees, however, stand in contrast to the
strong empirical performance of learned universal predictors
\citep{MGH+24,entropy2026} and, more broadly, to the outstanding performance of
LLMs across a wide range of tasks and benchmarks. 
This contrast raises the
question of why approximation has a milder effect in practice.
%This contrast raises the question of whether approximation has a milder effect under an evaluation criterion that more closely reflects empirical practice.
Empirical evaluations typically aggregate performance across many tasks or environments, whereas an environment-wise guarantee treats each environment individually. Accordingly, we first study prediction performance averaged over the same prior \(w\) that defines the training mixture \(\xi\). We show that, in sharp contrast to the worst-case setting, approximation has a substantially milder effect under this average-case criterion. Our distribution-shift theorem can then be used to extend this guarantee to other test distributions. 

%In this section, we present our results on average-case analysis of approximation robustness of sequential prediction. We consider these results to be the main results of this paper. Because states our theorems, we first discuss the motivation of this study. 

%Section~\ref{sec:bad} shows that the approximation-dependent term $\eps/w_\mu$ in Theorem~\ref{thm:upper} is unavoidablein general. For environments with small prior weights, this term can significantly weaken the $O(\log 1/w_{\mu})$ guarantee of the ideal predictor obtained by exact optimization. Yet empirical studies demonstrate promising performance of learned universal predictors \citep{MGH+24,entropy2026}. More broadly, LLMs have demonstrated strong performance across a wide range of benchmarks. This motivates a closer examination of how these pessimistic worst-case guarantees relate to the strong empirical performance. 

%summary of the results in this section. 

%deepen our understanding; 

%it also provides some insights between the a and w in.}

%Our key observation is that empirical performance is typically assessed by averaging across environments. We therefore first consider evaluation under the same prior $w$ that defines the training mixture $\xi$. It turns out that in sharp contrast to the worst-case setting, approximation has a much milder effect on average-case prediction performance.

\subsection{Sharp Contrast between Typical and Task-Wise Prediction Guarantees}
\begin{theorem}[Approximation robustness: average-case upper bound]\label{thm:AverageCase}
\label{thm:upper2}
If $D_{\mathrm{KL}}(\xi\Vert\rho)\leq\eps$, then
\[
\mathbb{E}_{\mu\sim w}\left[\error_n(\mu,\rho)
-
\error_n(\mu,\xi)\right]
\leq
(\ln 2)\,\eps.
\]
\end{theorem}
This theorem is a consequence of the following exact average-case decomposition:
\begin{proposition}[Exact average-case decomposition]
\label{ExactAverageCaseDecomposition}
For every predictor $\rho$ and every $n\geq1$,
\[
\mathbb{E}_{\mu\sim w}\error_n(\mu,\rho)
=
\mathbb{E}_{\mu\sim w}\error_n(\mu,\xi)
+
\error_n(\xi,\rho).\]
\end{proposition}

Theorem~\ref{thm:AverageCase} matches the exact prediction-error baseline of the ideal mixture $\xi$: the average cost of replacing $\xi$ with $\rho$ is $\error_n(\xi,\rho)$. Together with Theorem~\ref{lem:lower_universal}, this result shows that approximation can preserve aggregate performance while substantially degrading performance on individual tasks.

%small average excess error can coexist with large prediction errors in rare environments. It thus provides a theoretical mechanism through which 

\begin{remark}
\label{rem:averaging-insufficient}
Note that Theorem~\ref{thm:upper2} cannot be obtained simply by averaging
the bound in Theorem~\ref{thm:upper}. Instead, the proof exploits the mixture structure
to identify the average excess error exactly.
To see this, consider the case when $\M$ is infinite, the following average diverges:
\[
\mathbb{E}_{\mu\sim w}\frac{1}{w_\mu}
=
\sum_{\mu\in\M}w_\mu\frac{1}{w_\mu}
=
\sum_{\mu\in\M}1
=
+\infty.
\]
%For $\eps>0$, the averaged approximation term therefore diverges.
\end{remark}

%\begin{remark}[Other error criteria and  a fundamental differentiation]
%\label{rem:average-other-errors}
\textbf{A fundamental differentiation}: The role of approximation in Theorem~\ref{thm:upper2} also highlights an interesting distinction between exact Bayesian prediction and its approximate counterpart. The classical theory applies to cumulative squared prediction error, KL error, and Hellinger error. As we show in Section~\ref{sec:good}, the task-wise guarantees under approximation extend to all three error criteria. By contrast, the average-case decomposition in Proposition \ref{ExactAverageCaseDecomposition} holds only for cumulative squared prediction error and cumulative KL error. For KL error, the chain rule immediately yields the analogous decomposition.
$\mathbb{E}_{\mu\sim w}D_n(\mu\Vert\rho)
=
\mathbb{E}_{\mu\sim w}D_n(\mu\Vert\xi)
+
D_n(\xi\Vert\rho).$
Hence, if $D_{\mathrm{KL}}(\xi\Vert\rho)\leq\eps$, then
$\mathbb{E}_{\mu\sim w}D_n(\mu\Vert\rho)
\leq
\mathbb{E}_{\mu\sim w}D_n(\mu\Vert\xi)+\eps.$
For cumulative squared Hellinger error, neither the analogous decomposition nor a uniform additive $O(\eps)$ excess-error bound relative to $\xi$ holds in general. The mixture predictor $\xi$ need not minimize average squared Hellinger error; a counterexample is given in Appendix~\ref{app:hellinger-counterexample}.

\begin{remark} Prior work has shown that the average gains of knowledge distillation can coexist with degraded performance on underrepresented subgroups and tail classes \citep{distill1,distill2}. By viewing $\xi$ as the teacher model and $\rho$ as the student model, our results provide a complementary explanation in Bayesian sequential prediction: a small population KL distillation loss controls the student’s average excess prediction error, but can coexist with substantial degradation relative to the teacher in rare environments.
\end{remark}

\subsection{Approximation Robustness under Distribution Shift}
\label{sec:shift}

The preceding analysis evaluates performance under the training prior $w$. We now consider a different test prior $w'$. Our next theorem shows that the average-case guarantee remains robust when the test prior is dominated by the training prior, in the sense that $w_\nu' \leq Cw_\nu$ for every environment $\nu$ and some $C\geq 1$. This result has several implications beyond establishing a basic transfer principle for sequential prediction and interpolating between the worst-case and average-case analyses of robustness under approximation.
One potentially significant implication is that the theory provides basic support for the argument that augmenting real-world datasets with Universal-Turing-Machine-generated (UTM-generated) data can improve prediction performance on practically important families of tasks, as suggested by \citep{MGH+24}.

\begin{theorem}[Approximation robustness under distribution shift]
\label{thm:upper3}
Let $w'$ be a probability distribution on $\M$ satisfying
$w_\nu'\leq Cw_\nu$ for every $\nu\in\M$ and some $C\geq1$.
Define the corresponding mixture
\[
\xi':=\sum_{\nu\in\M}w_\nu'\nu.
\]
If $D_{\mathrm{KL}}(\xi\Vert\rho)\leq\eps$, then, for every
$n\geq1$,
\begin{equation}
\label{eq:shift-test-mixture}
\begin{aligned}
\mathbb{E}_{\mu\sim w'}\error_n(\mu,\rho)
&\leq
\mathbb{E}_{\mu\sim w'}\error_n(\mu,\xi')
+
(\ln 2)
\left(
\sqrt{\log C}+\sqrt{C\eps}
\right)^2\\
&\leq
\mathbb{E}_{\mu\sim w'}\error_n(\mu,\xi)
+
(\ln 2)
\left(
\sqrt{\log C}+\sqrt{C\eps}
\right)^2.
\end{aligned}
\end{equation}
In particular, the average excess error relative to $\xi'$ is $O(\log C+C\eps)$.
\end{theorem}

%\begin{remark}[Other error criteria]
%\label{rem:shift-other-errors}
Cumulative KL error satisfies a distribution-shift bound of the same order. Indeed, the KL decomposition and the domination $\xi'\leq C\xi$, together with the KL transfer inequality (Lemma \ref{lem:transfer}) used in the proof of Theorem~\ref{thm:upper}, give
\[
\mathbb{E}_{\mu\sim w'}D_n(\mu\Vert\rho)
\leq
\mathbb{E}_{\mu\sim w'}D_n(\mu\Vert\xi')
+
\frac{e}{e-1}
\left(\log C+C\eps\right).
\]

\paragraph{Implication I: Rare task families and targeted sampling.}
Theorem~\ref{thm:upper3} interpolates between average-case and
environment-wise guarantees.
Taking $w'=w$ and $C=1$ recovers the average-case upper bound
in Theorem~\ref{thm:upper2}.
At the other extreme, if $w'$ is concentrated on a single
environment $\mu$, we may take $C=1/w_\mu$.
Then $\xi'=\mu$, and
\[
\error_n(\mu,\rho)
\leq
(\ln 2)
\left(
\sqrt{\log\frac{1}{w_\mu}}
+
\sqrt{\frac{\eps}{w_\mu}}
\right)^2
=
O\!\left(
\log\frac{1}{w_\mu}+\frac{\eps}{w_\mu}
\right).
\]
Thus, the theorem also recovers the order of the
environment-wise bound in Theorem~\ref{thm:upper}.

More generally, let $A\subseteq\M$ be a family of environments
with prior mass $w(A)>0$, and consider evaluation under
$w'=w(\cdot\mid A)$.
Since $w_\nu'\leq w_\nu/w(A)$, we may take $C=1/w(A)$ and obtain
\begin{equation}
\label{eq:implicationiof14}
\mathbb{E}_{\mu\sim w'}\error_n(\mu,\rho)
\leq
\mathbb{E}_{\mu\sim w'}\error_n(\mu,\xi')
+
O\!\left(
\log\frac{1}{w(A)}
+
\frac{\eps}{w(A)}
\right).
\end{equation}
For this conditional test distribution, the guarantee depends on the total prior mass of the task family.

This suggests a way to improve coverage of an important but
underrepresented family.
Let $q$ be a probability distribution supported on that family,
and define
\[
\widetilde w:=(1-\lambda)w+\lambda q,
\qquad
\xi_q:=\sum_{\nu\in\M}q_\nu\nu,
\qquad
\widetilde\xi:=\sum_{\nu\in\M} \widetilde w_{\nu}\nu=(1-\lambda)\xi+\lambda\xi_q,
\]
where $\lambda\in(0,1)$ is the fraction of training mass
allocated to $q$.
Since $q_\nu\leq\widetilde w_\nu/\lambda$, any predictor with $D_{\mathrm{KL}}(\widetilde\xi\Vert\rho)\leq\eps$ satisfies
\[
\mathbb{E}_{\mu\sim q}\error_n(\mu,\rho)
\leq
\mathbb{E}_{\mu\sim q}\error_n(\mu,\xi_q)
+
O\!\left(
\log\frac{1}{\lambda}+\frac{\eps}{\lambda}
\right).
\]
The approximation guarantee for this family therefore depends
on its assigned training mass $\lambda$, independently of its
rarity under the original prior $w$.

Within this model, the bound provides a quantitative rationale for increasing the training probability assigned to important but underrepresented task families, for example through targeted sampling or synthetic data generation.
%Within this model, the bound provides a quantitative justification for targeted sampling and synthetic data generation.
In particular, it supports augmenting UTM-generated training data with real-world datasets to improve prediction performance on practically important task families, as suggested by \cite{MGH+24}.

\paragraph{Implication II: Fine-tuning with replay.}
Replaying pretraining data during adaptation can mitigate forgetting \citep{DBLP:conf/icml/BethuneGBGCA25,DBLP:journals/tmlr/IbrahimTGRABLR24}. Theorem \ref{thm:upper3} can be used to give a population-level analysis of this strategy. Interpret $q$ as the target-task distribution and $w$ as the original pretraining distribution.
During fine-tuning, draw a fraction $\lambda$ of sequences
from $\xi_q$ and a fraction $1-\lambda$ from $\xi$.
The resulting population training distribution is
$\widetilde\xi=(1-\lambda)\xi+\lambda\xi_q$.

Suppose the fine-tuned predictor $\rho_{\mathrm{ft}}$ approximately minimizes the population log-loss, in that
\[
D_{\mathrm{KL}}(\widetilde\xi\Vert\rho_{\mathrm{ft}})
\leq\eps.
\]
Both task distributions are dominated by the fine-tuning prior:
%\[
$q_\nu\leq\frac{\widetilde w_\nu}{\lambda}$, $w_\nu\leq\frac{\widetilde w_\nu}{1-\lambda}.$
%\]
Applying Theorem~\ref{thm:upper3} to each distribution yields
the simultaneous guarantees
\[
\begin{aligned}
\mathbb{E}_{\mu\sim q}\error_n(\mu,\rho_{\mathrm{ft}})
&\leq
\mathbb{E}_{\mu\sim q}\error_n(\mu,\xi_q)
+
O\!\left(
\log\frac{1}{\lambda}+\frac{\eps}{\lambda}
\right),\\
\mathbb{E}_{\mu\sim w}\error_n(\mu,\rho_{\mathrm{ft}})
&\leq
\mathbb{E}_{\mu\sim w}\error_n(\mu,\xi)
+
O\!\left(
\log\frac{1}{1-\lambda}+\frac{\eps}{1-\lambda}
\right).
\end{aligned}
\]
The first bound controls prediction error on the target distribution, while the second controls error on the original
distribution. Since $\xi$ minimizes average squared prediction error under $w$, the second bound also limits the possible increase in average error relative to the predictor before fine-tuning.

For fixed approximation accuracy $\eps$, increasing $\lambda$
tightens the target-distribution bound while weakening the
original-distribution bound. The mixing parameter $\lambda$ therefore controls the trade-off
between the performance guarantees on the target and original
task distributions.
Finally, we remark that these guarantees assume that the fine-tuned predictor achieves
the stated population approximation accuracy $\eta$.

\paragraph{Implication III: Choice of reference universal Turing machine.}
Let $U$ and $U'$ be two universal Turing machines. By the invariance theorem, there is a constant $c_{U,U'}\geq0$
such that
\[
\left|K_U(\nu)-K_{U'}(\nu)\right|
\leq c_{U,U'}
\qquad
\text{for every estimable $\nu$}.
\]
Consider the universal priors $w_\nu=2^{-K_U(\nu)}/Z_U$, $w_\nu'=2^{-K_{U'}(\nu)}/Z_{U'}$, where $Z_V:=\sum_{\nu\in\M}2^{-K_V(\nu)}$.
%\[
%w_\nu=\frac{2^{-K_U(\nu)}}{Z_U},\qquad w_\nu'=\frac{2^{-K_{U'}(\nu)}}{Z_{U'}}, \qquad Z_V:=\sum_{\nu\in\M}2^{-K_V(\nu)}. \]
Then we have $w_\nu'/w_\nu = 2^{K_U(\nu)-K_{U'}(\nu)}
\cdot Z_U/Z_{U'} \leq 4^{c_{U,U'}}$.
%\[
%\frac{w_\nu'}{w_\nu}
%=
%2^{K_U(\nu)-K_{U'}(\nu)}
%\frac{Z_U}{Z_{U'}}
%\leq
%4^{c_{U,U'}}.
%\]
If training uses the mixture $\xi$ induced by $U$ and
$D_{\mathrm{KL}}(\xi\Vert\rho)\leq\eps$, then test
under the prior induced by $U'$ satisfies
\[
\mathbb{E}_{\mu\sim w'}\error_n(\mu,\rho)
\leq
\mathbb{E}_{\mu\sim w'}\error_n(\mu,\xi')
+
O\!\left(
c_{U,U'}+4^{c_{U,U'}}\eps
\right).
\]
Thus, changing the reference machine preserves the form of
the average-case guarantee, with constants depending only
on the pair of machines, independently of the environment
and prediction horizon.
These constants may nevertheless be large because their
dependence on the compiler overhead is exponential.

\section{Robustness of Learning and Prediction with Finite Data}
\label{sec:intro_data}

We now examine whether our prediction guarantees remain robust when the predictor is learned from finitely many samples.
Fix a prediction horizon $n$, and let $\rho_m$ be learned from $m$ independent training sequences
$\mathcal D_m\sim(\xi_{1:n})^{\otimes m}$ by minimizing the empirical log-loss over the hypothesis class $\mathcal H$
up to an additive optimization error $\eta_{\mathrm{opt}}\geq 0$. In addition to this optimization error, the model class incurs a representation error $\eta_{\mathrm{rep}}:=\min_{\rho\in\mathcal H} \KL^n(\xi\Vert\rho)$.  We consider $\eta_{\mathrm{rep}}+\eta_{\mathrm{opt}}>0$.

We show an exponential upper-tail bound on the population KL error: for every $\gamma>0$, there exist constants
$c_\gamma,C_\gamma>0$ such that\footnote{The constants $C,c$ in \eqref{eq:kl-concentration}, \eqref{eq:s-finite} and \eqref{eq:theorem3-finite-robust} can depend on $\xi_{1:n}$ and $n$.}, for all $m\geq1$,
\begin{equation}\label{eq:kl-concentration}
\Pr_{\mathcal D_m} \left\{\KL^n(\xi\Vert\rho_m)>\eta_{\mathrm{rep}}+\eta_{\mathrm{opt}}+\gamma\right\} \leq C_\gamma e^{-c_\gamma m}.
\end{equation}
Nevertheless, this high-probability guarantee need not imply
a finite expected KL error. A learned predictor may assign
zero probability to an event that has positive population
probability but is absent from the training data. Precisely, it is possible that
\begin{equation}\label{eq:kl-infinite}
\mathbb E_{\mathcal D_m}
\KL^n(\xi\Vert\rho_m)=+\infty
\qquad
\text{for every finite }m.
\end{equation}
Despite this difficulty, we show that our squared prediction error guarantees remain
meaningful in expectation. It follows from two observations. First, if we characterize the approximation quality with a weaker condition $S_n(\xi,\rho_m)\leq \eps$ instead of $\KL^n(\xi\|\rho_m)\leq \eta$, then there exist constants $c$ and $C$ such that for any $m\geq 1$,
\begin{equation}\label{eq:s-finite}
\mathbb E_{\mathcal D_m} S_n(\xi,\rho_{m})\leq \eta_{\mathrm{rep}}+\eta_{\mathrm{opt}} +C\cdot n e^{-c\cdot m},
\end{equation}
which means that the finite-sample contribution to the approximation error decays exponentially with the training sample size. Moreover, one can check that all of our theorems %\ref{thm:upper}, \ref{lem:lower_universal}, \ref{thm:upper2}, \ref{thm:upper3} 
extend from the KL approximation condition to the weaker condition $S_n(\xi,\rho)\leq\eta$, with corresponding changes to the constants. In particular, Theorem \ref{thm:upper2} and \eqref{eq:s-finite} together imply: for any $m\geq 1$,
\begin{equation}\label{eq:theorem3-finite-robust}
\E_{\mathcal D_m}\big[\mathbb{E}_{\mu\sim w}\error_n(\mu,\rho_m)-\mathbb{E}_{\mu\sim w}\error_n(\mu,\xi)\big]\leq\eps_{\mathrm{rep}}+\eps_{\mathrm{opt}}+C\cdot n e^{-c\cdot m}.
\end{equation}

\section{Conclusion}\label{sec:conclusion}
We characterize how approximation affects Bayesian prediction and Solomonoff induction. Our results show that approximation can preserve average performance while substantially degrading performance on individual tasks. The analysis also extends to dominated shifts from the training prior to the test prior and is robust in the finite-sample setting.

An important direction for future work is to develop mitigation
strategies that provide stronger guarantees across all tasks
simultaneously. We explore initial approaches in
Appendices~\ref{sec:intro_mitigation} and~\ref{app:otherobjective}.

Another open direction is length generalization. Our finite-sample
analysis considers training and test sequences of the same length;
extending these guarantees to longer test sequences remains a
challenge. One possibility is to combine pretrained predictors with
continual learning, using minimum description length principle \citep{mdlmit,mdl1} to guide
adaptation as new observations become available.

\section*{Acknowledgments}
We thank Xinyu Mao and Guangxu Yang for early discussions, and Marcus Hutter for his comments. We used AI tools for polishing the language, searching the literature, and developing and
drafting all proofs. We have checked and revised all proofs and take full responsibility for the
whole content of this work.

\bibliography{reference}
\bibliographystyle{alpha}
\appendix

%%%%%%%%%%%%%%%%%%%%%%%%%%%%%%%%%%%%%%%%%%%%%%%%
\section{Proof of Theorem \ref{thm:upper}}
\label{Appendix:ProofTaskWiseUpperBound}

\begin{lemma}[KL transfer under domination]
\label{lem:transfer}
Let $P,Q,R$ be probability distributions on the same finite space $\Omega$.
Suppose that $Q\geq w\cdot P$ for some $w\in(0,1]$. Then
\[
D_{\mathrm{KL}}(P\Vert R)
\leq\frac{e}{e-1}
\left(\log\frac1w+\frac1w\cdot D_{\mathrm{KL}}(Q\Vert R)\right).
\]
\end{lemma}

\begin{proof}
The claim is immediate when $w=1$, because then $Q=P$. We may therefore
assume $0<w<1$. The conclusion is also immediate if
$D_{\mathrm{KL}}(Q\Vert R)=\infty$.

Write $Q=wP+(1-w)V$ for a probability distribution $V$, and define
\[
\alpha(x):=\frac{wP(x)}{Q(x)}
\]
on the support of $Q$, with $\alpha(x)=0$ elsewhere. Then
$0\leq\alpha\leq1$ and $\E_Q\alpha=w$. Introduce an auxiliary binary
variable $Z$. Under $\widetilde Q$, let $X\sim Q$ and
\[
Z\mid X=x\sim\operatorname{Bern}(\alpha(x)).
\]
Thus $\widetilde Q(X\in\cdot\mid Z=1)=P$. Define $\widetilde R$ using the
same conditional law of $Z$ given $X$, but with $X\sim R$, and put
\[
r:=\E_R\alpha=\Pr_{\tilde{R}} [Z=1].
\]
Finiteness of $D_{\mathrm{KL}}(Q\Vert R)$ implies $0<r<1$. Let $R_1$
denote the conditional law of $X$ under $\widetilde R$ given $Z=1$, and
define $R_0$ analogously. One can check that 
\begin{align*}
R_1(x)=\frac{\alpha(x)}{r}\cdot R(x),\quad R_0(x)=\frac{1-\alpha(x)}{1-r} \cdot R(x),\quad \mbox{and } R=r R_1+ (1-r) R_0.
\end{align*}

The conditional kernel from $X$ to $Z$ is the same under the two joint
distributions. Hence
\[
D_{\mathrm{KL}}(\widetilde Q\Vert\widetilde R)
=D_{\mathrm{KL}}(Q\Vert R).
\]
Applying the chain rule in the other direction gives
\begin{align}
D_{\mathrm{KL}}(Q\Vert R)
&=d_{\mathrm B}(w\Vert r)
  +wD_{\mathrm{KL}}(P\Vert R_1)
  +(1-w)D_{\mathrm{KL}}(V\Vert R_0)\notag\\
&\geq d_{\mathrm B}(w\Vert r)
  +wD_{\mathrm{KL}}(P\Vert R_1),
\label{eq:binary-chain}
\end{align}
where
\[
d_{\mathrm B}(u\Vert v)
:=u\log\frac uv+(1-u)\log\frac{1-u}{1-v}
\]
is binary KL divergence. Moreover,
\[
\frac{R_1(x)}{R(x)}=\frac{\alpha(x)}r,
\]
and therefore
\begin{align*}
D_{\mathrm{KL}}(P\Vert R)
&=D_{\mathrm{KL}}(P\Vert R_1)+\E_P[\log\alpha]-\log r\\
&\leq D_{\mathrm{KL}}(P\Vert R_1)-\log r.
\end{align*}
Together with \eqref{eq:binary-chain}, this yields
\begin{equation}
D_{\mathrm{KL}}(P\Vert R)
\leq
\frac{D_{\mathrm{KL}}(Q\Vert R)-d_{\mathrm B}(w\Vert r)}w
-\log r.
\label{eq:pre-transfer}
\end{equation}

We next use the inequality
\begin{equation}
-\log r
\leq\frac{e}{e-1}
\left(\log\frac1w+\frac1w d_{\mathrm B}(w\Vert r)\right),
\qquad 0<w,r<1.
\label{eq:scalar-transfer}
\end{equation}
To verify it, first observe that
\[
\log\frac1w+\frac1w d_{\mathrm B}(w\Vert r)
=-\log r+\frac{1-w}{w}\log\frac{1-w}{1-r}.
\]
If $r\geq w$, the second term is nonnegative, and
\eqref{eq:scalar-transfer} follows. If $r<w$, the desired inequality is
equivalent to
\[
e(1-w)\log\frac{1-r}{1-w}\leq-w\log r.
\]
For fixed $w$, let
\[
F_w(r):=-w\log r-e(1-w)\log\frac{1-r}{1-w}.
\]
This function is strictly convex on $(0,w)$, and its stationary point is
\[
r_*:=\frac{w}{e-(e-1)w}.
\]
Writing $a:=1-w$, direct substitution gives
\[
F_w(r_*)
=G(a):=[1+(e-1)a]\log[1+(e-1)a]
 -(1-a)\log(1-a)-ea\log e.
\]
Now $G(0)=G(1)=0$ and
\[
G'(a)=(e-1)\log[1+(e-1)a]+\log(1-a),
\]
while
\[
G''(a)=\frac{1}{\ln 2}\left(\frac{(e-1)^2}{1+(e-1)a}-\frac1{1-a}\right).
\]
The numerator of $G''$ after taking a common denominator is affine and
strictly decreasing, positive at $a=0$ and negative at $a=1$. Thus $G'$
first increases and then decreases. Since $G'(0)=0$ and
$G'(a)\to-\infty$ as $a\uparrow1$, it is initially positive and then crosses
zero exactly once. Hence $G$ first increases and then decreases to
$G(1)=0$. It follows that $G(a)\geq0$, proving
\eqref{eq:scalar-transfer}.

Let $c:=e/(e-1)$. Substituting \eqref{eq:scalar-transfer} into
\eqref{eq:pre-transfer}, and using
$d_{\mathrm B}(w\Vert r)\leq D_{\mathrm{KL}}(Q\Vert R)$, gives
\begin{align*}
D_{\mathrm{KL}}(P\Vert R)
&\leq\frac{D_{\mathrm{KL}}(Q\Vert R)}w
 +(c-1)\frac{d_{\mathrm B}(w\Vert r)}w
 +c\log\frac1w\\
&\leq c\left(
\log\frac1w+\frac1wD_{\mathrm{KL}}(Q\Vert R)
\right).\qedhere
\end{align*}
\end{proof}

\begin{proof}[Proof of Theorem~\ref{thm:upper}]
Fix $n$, and apply Lemma~\ref{lem:transfer} to 
\[
P=\mu_{1:n},\qquad Q=\xi_{1:n},\qquad R=\rho_{1:n}.
\]
Since $\xi\geq w_\mu\mu$,
\[
D_{\mathrm{KL}}^n(\mu\Vert\rho)
\leq\frac{e}{e-1}
\left(
\log\frac1{w_\mu}
+\frac{D_{\mathrm{KL}}^n(\xi\Vert\rho)}{w_\mu}
\right)\leq \frac{e}{e-1}
\left(
\log\frac1{w_\mu}
+\frac{\eps}{w_\mu}
\right).
\]
Applying
Lemma~\ref{lem:pinsker} and taking the supremum over $n$ gives the conclusion.
\end{proof}

\section{Proof of Lemma \ref{lem:lower}}\label{sec:proof3}
Let $m:=|y|$ and set $N:=\left\lfloor \eps/q\right\rfloor$. Note that
\[
q=\xi(y)\geq w_\mu\mu(y)=w_\mu>0.
\]

For $0\leq \ell<N$, write
\[
h_\ell:=z_{1:m+\ell}.
\]
We construct a measure $\rho$ on $\{0,1\}^\infty$ as follows: set
\[
\rho(\cdot\mid h_\ell):=\operatorname{Bern}(1/2),
\qquad 0\leq\ell<N,
\]
and let $\rho$ agree with $\xi$ at every other history of positive
$\xi$-probability. At histories of zero $\xi$-probability, choose an
arbitrary conditional distribution. 

For any distribution $p$ on $\{0,1\}$,
\[
D_{\mathrm{KL}}(p\|\operatorname{Bern}(1/2))=1-H(p) \leq 1,
\]
where $H(p)$ is the binary entropy. Hence, for every $n$, the chain rule gives
\begin{align*}
D_{\mathrm{KL}}^n(\xi\|\rho)\leq \sum_{0\leq \ell <N}\xi(h_\ell)
D_{\mathrm{KL}}\left(\xi(\cdot\mid h_\ell)\Vert \operatorname{Bern}(1/2)
\right)\leq \sum_{0\leq \ell <N}\xi(h_\ell)\leq Nq. 
\end{align*}
Taking the supremum over $n$ yields
\[
D_{\mathrm{KL}}(\xi\Vert\rho)\leq\eps.
\]

The modified conditionals occur only after the prefix $y$. Therefore, the
first $m$ predictions along $z$ are unchanged, and
\[
S_m(\mu,\rho)=S_m(\mu,\xi).
\]
At each history $h_\ell$, the next-symbol distribution under
$\mu=\delta_z$ is a point mass, whereas
$\rho(\cdot\mid h_\ell)=\operatorname{Bern}(1/2)$. Thus
\[
\left(\mu(1\mid h_\ell)-\rho(1\mid h_\ell)\right)^2+\left(\mu(0\mid h_\ell)-\rho(0\mid h_\ell)\right)^2=\frac{1}{2}.
\]
Consequently,
\[
\error(\mu,\rho)
\geq S_{m+N}(\mu,\rho)
=S_m(\mu,\xi)+\frac{N}{2}.
\]
This finishes the proof.

\section{Proof of Theorem \ref{lem:lower_universal}}
\label{AppendixProofLowerBound}
We remark Lemma~\ref{lem:lower} already gives a nearly exponential lower bound in the universal setting. Specifically, for every $n$, since $\sum_{h\in\{0,1\}^n}\xi_U(h)=1$, there is a string $h_n\in\{0,1\}^n$ such that $\xi_U(h_n)\leq2^{-n}$. Let $\mu_n:=\delta_{h_n0^\infty}$. We have 
\[
n-O(1)\leq K(\mu_n)\leq n+O(\log n).
\]
The first inequality is by combining (i) $\xi_U(h_n)\leq 2^{-n}$ and (ii) $\xi_U(h_n)\geq w_{\mu_n}=\Theta(2^{-K(\mu_n)})$. Lemma~\ref{lem:lower} therefore yields, for all sufficiently large $n$,
\[
\sup_{\rho:D_{\mathrm{KL}}(\xi_U\Vert\rho)\leq\eps}
\error(\mu_n,\rho)=\Omega\left(
\eps\frac{2^{K(\mu_n)}}{K(\mu_n)^{O(1)}}
\right).
\]
This matches the upper bound in Theorem~\ref{thm:upper} up to a polynomial factor. We next show that this factor can be removed.

We will make use of the following coding theorem.
\begin{lemma}[Corollary 4.3.1 in \citep{li2008introduction}]\label{lem:coding}
If $P$ is a lower semi-computable semi-measure on $\{0,1\}^*$ with $\sum_{y\in \{0,1\}^*} P(y)\leq 1$, then there is a
constant $c_P=K(P)+O(1)$ such that $K(x)\leq -\log P(x) + c_P$.
\end{lemma}

\begin{proof}[Proof of Theorem \ref{lem:lower_universal}]
For $m\geq1$, define
\[
\mathcal C_m
:=
\left\{
1^m0s:s\in\{0,1\}^m
\right\},
\qquad
\mathcal C:=\bigcup_{m\geq1}\mathcal C_m.
\]
Note that the set $\mathcal C$ is prefix-free. Since $|\mathcal C_m|=2^m$, a counting argument gives a string
$y_m\in\mathcal C_m$ satisfying $K(y_m)\geq m$. Here 
\[
K(y):=\min\{|p|:U(p)=y\}
\]
is the prefix Kolmogorov complexity of the string $y$.  Denote $\mu_m:=\delta_{y_m0^{\infty}}$ the deterministic measure outputting $y_m0^\infty$. Note that $\mu_m$ is estimable, as it can hard-code the self-delimiting program generating $y_m$ and then output zeros forever. In addition, we have $K(\mu_m)\leq K(y_m)+O(1)$.

Let $M_U$ be the Solomonoff prior
\[
M_U(y):=\sum_{p:U(p)=y*}2^{-|p|},
\]
which is semicomputable. Since $\mathcal C$ is prefix-free, $\sum_{y\in C} M_U(y)\leq 1$. By letting $P(y)=M_U(y)\cdot \vec{1}_{y\in C}$, Lemma \ref{lem:coding} therefore gives
\[
M_U(y)=O\left(2^{-K(y)}\right)
\qquad
\text{for every }y\in\mathcal C,
\]
Furthermore, a standard fact in algorithmic information theory (see e.g. Section 1.2.5 in \citep{uai}) says that $\xi_U(y)=O (M_U(y))$. Thus,
\[
\xi_U(y)=O\left(2^{-K(y)}\right)
\qquad
\text{for every }y\in\mathcal C,
\]
In particular, 
\[
\xi_U(y_m)=O\left(2^{-K(y_m)}\right)=O\left(2^{-K(\mu_m)}\right).
\]

Fix $\eps>0$. Applying Lemma~\ref{lem:lower} with
$\mu=\mu_m$ and $y=y_m$ gives:  there exists a universal constant $C_U>0$, such that
\begin{align*}
\sup_{\rho:\,D_{\mathrm{KL}}(\xi_U\Vert\rho)\leq\eps}
\error(\mu_m,\rho)\geq
\frac12
\left\lfloor
\frac{\eps2^{K(\mu_m)}}{C_U}
\right\rfloor\geq \frac{1}{4C_U}
\left(
K(\mu_m)+\eps2^{K(\mu_m)}
\right),
\end{align*}
for infinitely many $m$.
\end{proof}

\section{Proof of Proposition \ref{ExactAverageCaseDecomposition}}\label{sec:proof2}
For a history $h$ with $\xi(h)>0$, let $w(\mu\mid h):=w_\mu\mu(h)/\xi(h)$ denote the posterior weight of $\mu$. 
%For a symbol $a\in\X$, denote
%\[
%p_\mu(a\mid h):=\mu(a\mid h),\qquad
%p^a_\xi(h):=\xi(a\mid h),\qquad
%p^a_\rho(h):=\rho(a\mid h).
%\]
For any symbol $a\in \X$,
\[
\xi(a\mid h)=\sum_{\mu\in\M}w(\mu\mid h)\mu(a\mid h).
\]
%Since the alphabet is binary,
%\[
%d_{\mathrm{TV}}^2\bigl(\mu(\cdot\mid h),\rho(\cdot\mid h)\bigr)=\bigl(p_\mu(h)-p_\rho(h)\bigr)^2.
%\]
The variance decomposition gives %(\textbf{missing a factor of 2?})
\begin{align*}
\sum_{\mu\in\mathcal M}w(\mu\mid h)\bigl(\mu(a\mid h)-\rho(a\mid h)\bigr)^2
=\sum_{\mu\in\mathcal M}w(\mu\mid h)\bigl(\mu(a\mid h)-\xi(a\mid h)\bigr)^2+\bigl(\xi(a\mid h)-\rho(a\mid h)\bigr)^2.
\end{align*}
%\textbf{Did we return to the general alphabet case? do we really need to first go with the binary case. I think we should either prove directly for general alphabet.}
Multiplying by $\xi(h)$ and using
$\xi(h)w(\mu\mid h)=w_\mu\mu(h)$ yields
\begin{align*}
\sum_{\mu\in\mathcal M}w_\mu\mu(h)\bigl(\mu(a\mid h)-\rho(a\mid h)\bigr)^2
=\sum_{\mu\in\mathcal M}w_\mu\mu(h)\bigl(\mu(a\mid h)-\xi(a\mid h)\bigr)^2
+\xi(h)\bigl(\xi(a\mid h)-\rho(a\mid h)\bigr)^2.
\end{align*}
Summing over all histories $h\in\X^{t-1}$, all symbols $a\in\X$, and all
$1\leq t\leq n$ gives
\[
\mathbb E_{\mu\sim w} \error_n(\mu,\rho)=\mathbb E_{\mu\sim w} \error_n(\mu,\xi)+\error_n(\xi,\rho).
\]
%where the last inequality is due to Lemma \ref{lem:pinsker}.

\section{A counterexample for squared Hellinger error}\label{app:hellinger-counterexample}

Let $\mu_1=\delta_{1^\infty}$ and $\mu_0=\delta_{0^\infty}$, with prior weights $w_1=3/4$ and $w_0=1/4$, respectively.
Write $q_r=r\mu_1+(1-r)\mu_0$, so that $\xi=q_{3/4}$. After the first symbol $x_1$, every $q_r$ with $r\in(0,1)$ predicts the remaining sequence perfectly. Thus its averaged cumulative squared Hellinger error is
\[
R(r):=\mathbb{E}_{\mu\sim w} \Hellinger(\mu,q_r)=2-2\bigl[w_1\sqrt r+w_0\sqrt{1-r}\bigr].
\]
This is minimized at
\[
r_*=\frac{w_1^2}{w_1^2+w_0^2}=\frac9{10}\neq \frac{3}{4}.
\]
Hence $\xi$ is not a minimizer, ruling out the analogous
decomposition into its average error and a nonnegative
approximation term.

Moreover, let $\rho_u=q_{3/4-u}$ for $u>0$ sufficiently small. Since $\rho_u$ and $\xi$ differ only in their first-step
predictions,
\[
\eps_u:=D_{\mathrm{KL}}(\xi\Vert\rho_u)
=w_1\log\frac{w_1}{w_1-u}
 +w_0\log\frac{w_0}{w_0+u}
=\Theta(u^2).
\]
However,
\[
R(3/4-u)-R(3/4)
=\frac{\sqrt3-1}{2}\,u+O(u^2)
=\Theta(\sqrt{\eps_u}).
\]
Thus, no finite constant $C$ can make
$R(3/4-u)\le R(3/4)+C\eps_u$ hold for all sufficiently small $u$.

\section{Proof of Theorem \ref{thm:upper3}}\label{sec:upper3}
By Proposition~\ref{ExactAverageCaseDecomposition}, we have
\begin{equation}\label{eq:shift-variance}
\mathbb E_{\mu\sim w'}\error_n(\mu,\rho)
=
\mathbb E_{\mu\sim w'}\error_n(\mu,\xi')
+\error_n(\xi',\rho),
\end{equation}
and
\[
\mathbb E_{\mu\sim w'}\error_n(\mu,\xi)
=
\mathbb E_{\mu\sim w'}\error_n(\mu,\xi')+\error_n(\xi',\xi)
\geq
\mathbb E_{\mu\sim w'}\error_n(\mu,\xi').
\]
In the following, we will show that 
\begin{equation}\label{eq:xiprimerho}
S_n(\xi',\rho)\leq (\ln 2)\left(
\sqrt{\log C}+\sqrt{C\eps}
\right)^2,
\end{equation}
which concludes the proof.

The assumption $w_\nu'\leq Cw_\nu$ implies that, for every finite
history $h$,
\[
\xi'(h)
=
\sum_{\nu\in\M}w_\nu'\nu(h)
\leq
C\sum_{\nu\in\M}w_\nu\nu(h)
=
C\xi(h).
\]
In particular, $D_{\mathrm{KL}}^n(\xi'\Vert\xi)\leq\log C$. Hence, by Lemma \ref{lem:pinsker}, we have
\begin{equation}
\label{eq:shift-xiprime-xi}
\error_n(\xi',\xi)
\leq
(\ln2)\cdot D_{\mathrm{KL}}^n(\xi'\Vert\xi)
\leq (\ln2)\cdot \log C.
\end{equation}

Next, define
\[
B_n
:=
\sum_{t=1}^n
\sum_{h\in\X^{t-1}}
\xi'(h)\sum_{a\in\X} \left(\xi(a\mid h)-\rho(a\mid h)\right)^2.
\]
Using $\xi'(h)\leq C\xi(h)$ and Lemma \ref{lem:pinsker}, we obtain
\begin{equation}\label{eq:B}
B_n \leq C\sum_{t=1}^n
\sum_{h\in\X^{t-1}}
\xi(h)
\sum_{a\in\X} \left(\xi(a\mid h)-\rho(a\mid h)\right)^2=C\cdot \error_n(\xi,\rho)\leq (C\ln 2)\cdot D_{\mathrm{KL}}^n(\xi\Vert\rho)
\leq (\ln 2)\cdot C\eps.
\end{equation}

Expanding the square, we obtain
\[
\error_n(\xi',\rho)
=
\error_n(\xi',\xi)+B_n+2T_n,
\]
where
\[
T_n
:=
\sum_{t=1}^n
\sum_{h\in\X^{t-1}}
\xi'(h)
\sum_{a\in\X}
\bigl(\xi'(a\mid h)-\xi(a\mid h)\bigr)
\bigl(\xi(a\mid h)-\rho(a\mid h)\bigr).
\]
By the Cauchy--Schwarz inequality,
\[
|T_n|
\leq
\sqrt{\error_n(\xi',\xi)B_n}.
\]
It follows that
\[
\begin{split}
\error_n(\xi',\rho)
&\leq
\error_n(\xi',\xi)+B_n
+2\sqrt{\error_n(\xi',\xi)B_n}\\
&=
\left(
\sqrt{\error_n(\xi',\xi)}
+
\sqrt{B_n}
\right)^2.
\end{split}
\]
Together with \eqref{eq:shift-xiprime-xi} and \eqref{eq:B}, this yields \eqref{eq:xiprimerho}.

\section{Proofs for Section \ref{sec:intro_data}}\label{app:finite-data}
Fix a horizon $n$. Write $P:=\xi_{1:n}$, $H:=H(P)$, and $a:=\eta_{\mathrm{rep}}+\eta_{\mathrm{opt}}$.
Let $\rho^*\in \mathcal H$ denote the predictor achieving $\KL^n(\xi\| \rho^*)=\eta_{\mathrm{rep}}$.
Write $\widehat R_m$ for the empirical log-loss, so that $\widehat R_m(\rho_m)\leq\widehat R_m(\rho^*)+\eta_{\mathrm{opt}}$.

\paragraph{Proof of \eqref{eq:kl-concentration}.}
Let $\widehat P_m$ be the empirical distribution of the training
sequences. For $0<\tau<1$, define
\[
E_{m,\tau}:= \left\{\exists x\in\operatorname{supp}(P): |\widehat P_m(x)-P(x)|>\tau P(x) \right\}.
\]
Write $d:=|\operatorname{supp}(P)|$ and
\[
p_{\min}:=\min_{x\in\operatorname{supp}(P)}P(x)>0,
\]
Hoeffding's inequality and a union bound give
\[
\Pr[E_{m,\tau}]\leq 2d\exp(-2m\tau^2p_{\min}^2).
\]
Conditioned on $\overline{E}_{m,\tau}$, the nonnegativity of log-loss and approximate
empirical optimality imply
\[
(1-\tau)\bigl[H+\KL^n(\xi\Vert\rho_m)\bigr]
\leq\widehat R_m(\rho_m)
\leq(1+\tau)(H+\eta_{\mathrm{rep}})+\eta_{\mathrm{opt}}.
\]
Thus,
\[
\KL^n(\xi\Vert\rho_m)
\leq a+
\frac{\tau}{1-\tau}
\bigl[2(H+\eta_{\mathrm{rep}})+\eta_{\mathrm{opt}}\bigr].
\]
For any $\gamma>0$, choose $\tau>0$ small enough that the second term is at most $\gamma$.
The claim follows with $C_\gamma=2d$ and $c_\gamma=2\tau^2p_{\min}^2$.

\paragraph{Example establishing \eqref{eq:kl-infinite}.}
Take $n=1$, $P=\mathrm{Bernoulli}(1/2)$, and $\mathcal H$ consisting of $\rho_0:=\delta_0$ and $\rho_1:=\mathrm{Bernoulli}(1/4)$. %Then $\eta_{\mathrm{rep}}=\log(4/3)/2>0$.
On the event that all $m$ training observations are zero,
which has probability $2^{-m}$, exact empirical log-loss
minimization selects $\rho_m=\rho_0$.
Since $\KL(P\| \rho_0)=+\infty$, the expected KL error is
infinite for every finite $m$. Padding deterministic zeros extends the example to any fixed $n$.

\paragraph{Proof of \eqref{eq:s-finite}.}
Recall $a=\eta_{\mathrm{rep}}+\eta_{\mathrm{opt}}$, let $\gamma_a:=(1/\ln2-1)a$. 
By $0\leq S_n\leq2n$ and $S_n(\xi,\rho_m)\leq(\ln2)\KL^n(\xi\|\rho_m)$, we have
\begin{align*}
\mathbb E_{\mathcal D_m}
\bigl[(S_n(\xi,\rho_m)-a)_+\bigr]
\leq2n\Pr\{\KL^n(\xi\Vert\rho_m)>a+\gamma_a\}
\leq2nC_{\gamma_a}e^{-c_{\gamma_a}m}.
\end{align*}
%Hence $\sum_{m\geq1}b_m<\infty$ and$\mathbb E_{\mathcal D_m}S_n(\xi,\rho_m)\leq a+b_m$. The condition $a>0$ cannot generally be dropped: for a fair Bernoulli source, the empirical-frequency predictor is an exact log-loss minimizer with $a=0$, but $\mathbb E S_1(P,\rho_m)=1/(2m)$, which is not summable.

\paragraph{Theorems~\ref{thm:upper2}, \ref{thm:upper3},
and~\ref{thm:upper} under squared-error approximation.}
Suppose $S_n(\xi,\rho)\leq\eta$. By Proposition \ref{ExactAverageCaseDecomposition}, we have
\[
\mathbb E_{\mu\sim w}S_n(\mu,\rho)=\mathbb E_{\mu\sim w}S_n(\mu,\xi)+S_n(\xi,\rho)\leq \mathbb E_{\mu\sim w}S_n(\mu,\xi)+\eta.
\]
In the proof of Theorem \ref{thm:upper3}, by almost the same argument except adjusting \eqref{eq:B} to $B_n\leq C\eta$, we will get
\[
\mathbb E_{\mu\sim w'}S_n(\mu,\rho)
\leq\mathbb E_{\mu\sim w'}S_n(\mu,\xi')
+\left(\sqrt{(\ln2)\log C}+\sqrt{C\eta}\right)^2.
\]
Letting $w'$ concentrate on a single environment $\mu$, we recover Theorem~\ref{thm:upper}, up to constants:
\[
S_n(\mu,\rho)
\leq\left(
\sqrt{(\ln2)\log\frac1{w_\mu}}
+\sqrt{\frac{\eta}{w_\mu}}
\right)^2
=O\left(\log\frac1{w_\mu}+\frac{\eta}{w_\mu}\right).
\]

\paragraph{Theorem~\ref{lem:lower_universal} under squared-error approximation.}
By Lemma~\ref{lem:pinsker},
\[
\{\rho:\KL(\xi_U\Vert\rho)\leq\eta/\ln2\}
\subseteq
\{\rho:S(\xi_U,\rho)\leq\eta\}.
\]
Thus Theorem~\ref{lem:lower_universal}, applied with budget
$\eta/\ln2$, implies the same infinite-horizon lower bound
under squared-error approximation, up to constants.

\section{Mitigation With Category Information}\label{sec:intro_mitigation}

The worst-case bounds of order $\Theta(\eta/w_\mu)$ in Theorem~\ref{thm:upper} arise when approximation is controlled only
through a single aggregate KL bound $\KL(\xi\|\rho)\leq \eps$. When available, training-time task category information enables separate control of approximation errors within each task category, potentially improving category-wise prediction
performance. Notably, the learned predictor does not require category tags at test time.

Let $\mathcal T=\{T_1,\ldots,T_g\}$ be a finite set of categories, with $\operatorname{tag}:\M\to\mathcal T$ defining the partition $\M_j:=\{\mu\in\M:\operatorname{tag}(\mu)=T_j\}$.
Write
\[
p_j:=\sum_{\mu\in\M_j}w_\mu,
\qquad
\xi_j:=\sum_{\mu\in\M_j}\frac{w_\mu}{p_j}\mu.
\]
Thus $\xi=\sum_{j=1}^g p_j\xi_j$. Recall that, by Theorem~\ref{thm:upper}, an aggregate approximation
guarantee $\KL(\xi\Vert\rho)\leq\eta$ gives
\[
\error_n(\xi_j,\rho)=O\left(\log\frac1{p_j}+\frac{\eta}{p_j}\right).
\]
Direct control of the category-wise KL error avoids this
inverse-category-mass amplification.  The following guarantee follows immediately from
Lemma~\ref{lem:pinsker} and Proposition~\ref{ExactAverageCaseDecomposition}.

\begin{proposition}[Category-wise approximation]
\label{prop:category-approximation}
If $\KL^n(\xi_j\Vert\rho)\leq\eta_j$, then
\begin{align*}
\mathbb E_{\mu\sim w(\cdot\mid j)}\error_n(\mu,\rho)
-\mathbb E_{\mu\sim w(\cdot\mid j)}\error_n(\mu,\xi_j)=\error_n(\xi_j,\rho)
\leq(\ln2)\eta_j.
\end{align*}
\end{proposition}

In particular, if the tag identifies the generating task and
$\KL^n(\xi_\mu\Vert\rho)\leq\eta_\mu$, then
\[
\error_n(\mu,\rho)
=O\left(\eta_\mu\right).
\]

To control the category-wise approximation errors
$\KL^n(\xi_j\Vert\rho)$, one may minimize the following
worst-category excess log-loss, or a suitable surrogate:
\[
\min_\rho\max_{1\leq j\leq g}
\KL^n(\xi_j\Vert\rho).
\]

%Category tags also allow the proportions of training samples drawn
%from different categories to be adjusted, thereby trading off
%approximation errors across categories. Specifically, let
%$q=(q_1,\ldots,q_g)$ be another probability distribution over
%categories, with $\sum_jq_j=1$. Replacing the training category
%distribution $p$ with $q$ corresponds to minimizing
%\[
%\mathcal L_{q,n}(\rho)
%:=
%\sum_{j=1}^g q_j\,
%\mathbb E_{x_{1:n}\sim\xi_j}[-\log\rho(x_{1:n})]
%=
%\mathbb E_{x_{1:n}\sim\xi_q}[-\log\rho(x_{1:n})],
%\]
%where $\xi_q:=\sum_jq_j\xi_j$.
%Its excess above the entropy baseline $H(\xi_{q,1:n})$ equals
%$\KL^n(\xi_q\Vert\rho)$. An aggregate excess of at most $\eta$
%implies $\KL^n(\xi_j\Vert\rho)\leq\eta/q_j$.

%The distinction between exact and approximate learning is important. With exact learning, $q_\mu=\mu$, so $\rho_{\mathrm{tag}}=\xi$: tags provide no improvement over an already exact Bayesian mixture. Their potential benefit arises under approximation, by enabling training objectives or constraints that protect individual tasks instead of controlling only an average error.

\section{Higher moments of sequence log-loss} \label{app:otherobjective}
One may instead minimize the following sample-only training objective:
\[
\mathcal L_{p}(\rho):=\E_{x\sim\xi}\left[\bigl(-\log \rho(x)\bigr)^p\right],\quad p\geq 2.
\]
Fix a horizon $n$. Let $q$ be the minimizer of $\mathcal L_{p}$ among all probability distributions on $\X^n$.
Suppose the predictor $\rho$ satisfies
\[
\mathcal L_{p}(\rho_{1:n}) \leq \mathcal L_{p}(q)+\eps.
\]
\begin{lemma}\label{lem:otherobjective} For every $\mu\in\M$,
\[
\error_n(\mu,\rho)=O_{p,|\X|}\left(\log \frac{1}{w_\mu}+\log(n+1)+\left(\frac{\eps}{w_\mu}\right)^{1/p}\right).
\]
Here, the implicit constant in $O_{p,|\X|}(\cdot)$ may depend on $p$ and $|\X|$, but not on any other parameters.
\end{lemma}
\begin{proof}
Let $\Omega_n:=\{x\in\X^n:\xi(x)>0\}$ be the support of $\xi_{1:n}$. Since $\xi\geq w_\mu\mu$, the distribution
$\mu_{1:n}$ is also supported on $\Omega_n$.
%Write
%\[q(x):=\rho^\star(x),\qquadr(x):=\rho_{1:n}(x),\qquad \ell(x):=-\log q(x). \]
The minimizer $q$ assigns all its mass to $\Omega_n$, since otherwise, renormalizing its restriction to $\Omega_n$ would strictly decrease the objective. If $|\Omega_n|=1$, then the result follows immediately from
\[
\bigl(-\log \rho(x)\bigr)^p\leq\eps
\]
on the unique point $x\in\Omega_n$. We henceforth assume
$|\Omega_n|\geq2$. In this case, $q(x)\in(0,1)$ for every $x\in\Omega_n$.

Write $\ell(x):=-\log q(x)$. Consider the Lagrangian
\[
\mathscr L(q,\lambda)
:=
\sum_{x\in\Omega_n}
\xi(x)\bigl(-\log_2q(x)\bigr)^p
+
\lambda\left(\sum_{x\in\Omega_n}q(x)-1\right).
\]
Since the minimizer is an interior point, the first-order condition gives
\[
\frac{\partial \mathscr L}{\partial q(x)}=-\frac{p\,\xi(x)\ell(x)^{p-1}}
       {q(x)\ln 2}
+\lambda
=0,
\qquad x\in\Omega_n.
\]
Thus, setting $c_n:=\lambda\ln 2/p$, we obtain
\begin{equation}\label{eq:lp-optimality}
\xi(x)\ell(x)^{p-1}=c_nq(x),
\qquad x\in\Omega_n.
\end{equation}
Since $\xi(x)>0$, $\ell(x)>0$, and $q(x)>0$, we also have $c_n>0$. Summing \eqref{eq:lp-optimality} over $x$ gives
\[
c_n
=
\E_{x\sim\xi}\left[\ell(x)^{p-1}\right].
\]
Since $u\mapsto u^{(p-1)/p}$ is concave, Jensen's inequality gives
\begin{align*}
c_n&\leq\left(\E_{x\sim\xi}\left[\ell(x)^p\right]\right)^{(p-1)/p}=\mathcal L_p(q)^{(p-1)/p}\leq
\mathcal L_p\bigl(\operatorname{Unif}(\Omega_n)\bigr)^{(p-1)/p}\leq
\bigl(n\log|\X|\bigr)^{p-1}.
\end{align*}

We next show that $q$ dominates $\xi$. If $\ell(x)\geq1$, then
\eqref{eq:lp-optimality} gives
\[
\frac{q(x)}{\xi(x)}
=
\frac{\ell(x)^{p-1}}{c_n}
\geq
\frac{1}{\bigl(n\log|\X|\bigr)^{p-1}}.
\]
If $\ell(x)<1$, then $q(x)>1/2$, and hence
\[
\frac{q(x)}{\xi(x)}
\geq
q(x)
>
\frac12,
\]
The two cases imply
\begin{equation}
\label{eq:lp-domination}
q(x)
\geq
\frac{\xi(x)}{C_{n,p}},
\qquad
C_{n,p}:=2\bigl(n\log|\X|\bigr)^{p-1}.
\end{equation}

We now use the approximate optimality of $\rho$. Define $z(x):=\log(q(x)/\rho(x))$.
Since $\mathcal{L}_p(\rho_{1:n})<\infty$, it has that $\rho(x)>0$ on $\Omega_n$, and thus
\[
-\log \rho(x)=\ell(x)+z(x)\geq 0.
\]

For $p\geq2$ and $u,v\geq0$, we have
\begin{equation}
\label{eq:power-remainder}
v^p
\geq
u^p+pu^{p-1}(v-u)+|v-u|^p.
\end{equation}
Indeed, when $v=u+s$ with $s\geq0$, subtracting the right-hand side
from the left-hand side gives a function $F_+(s)$ satisfying
$F_+(0)=0$ and
\[
F_+'(s)
=
p\left(
(u+s)^{p-1}-u^{p-1}-s^{p-1}
\right)
\geq0.
\]
When $v=u-s$ with $0\leq s\leq u$, the corresponding function satisfies
$F_-(0)=0$ and
\[
F_-'(s)
=
p\left(
u^{p-1}-(u-s)^{p-1}-s^{p-1}
\right)
\geq0.
\]

Applying \eqref{eq:power-remainder} with
\[
u=\ell(x),
\qquad
v=-\log \rho(x)=\ell(x)+z(x),
\]
we obtain
\begin{align*}
\mathcal L_p(\rho)-\mathcal L_p(q)
=
\E_{x\sim\xi}
\left[
(\ell(x)+z(x))^p-\ell(x)^p
\right] \geq
\E_{x\sim\xi}\left[|z(x)|^p\right]
+
p\E_{x\sim\xi}
\left[
\ell(x)^{p-1}z(x)
\right].
\end{align*}
By \eqref{eq:lp-optimality},
\begin{align*}
\E_{x\sim\xi}
\left[
\ell(x)^{p-1}z(x)
\right]
=
c_n\sum_{x\in\Omega_n}
q(x)\log\frac{q(x)}{\rho(x)} =
c_nD_{\mathrm{KL}}(q\Vert \rho_{1:n})
\geq0.
\end{align*}
It follows from the approximation guarantee that
\begin{equation}
\label{eq:lp-ratio-moment}
\E_{x\sim\xi}
\left[
\left|
\log\frac{q(x)}{\rho(x)}
\right|^p
\right]
\leq
\eps.
\end{equation}

Since $\xi(x)\geq w_\mu\mu(x)$, 
\eqref{eq:lp-domination} and \eqref{eq:lp-ratio-moment} imply
\[
q(x)
\geq
\frac{w_\mu}{C_{n,p}}\mu(x)
\]
and
\[
\E_{x\sim\mu}\left[|z(x)|^p\right]
\leq
\frac{\eps}{w_\mu}.
\]
Therefore, by Holder's inequality,
\begin{align*}
D_{\mathrm{KL}}^n(\mu\Vert\rho)&=\E_{x\sim\mu}
\left[
\log\frac{\mu(x)}{q(x)}+z(x)
\right] \leq
\log\frac{C_{n,p}}{w_\mu}
+
\E_{x\sim\mu}|z(x)| \leq
\log\frac{C_{n,p}}{w_\mu}
+
\left(\frac{\eps}{w_\mu}\right)^{1/p}\\
&=O_{p,|\X|}\left(\log \frac{1}{w_\mu}+\log(n+1)+\left(\frac{\eps}{w_\mu}\right)^{1/p}\right).
\end{align*}
By applying Lemma~\ref{lem:pinsker}, we obtain the lemma.
\end{proof}

\end{document}